\documentclass[journal,10pt]{IEEEtran}

\usepackage{cite}
\usepackage{hyperref}
\usepackage{url}

\usepackage[pdftex]{graphicx}
\graphicspath{{./}}
\usepackage{amsmath,amssymb,amsfonts,amsthm,amscd,bm}

\usepackage{comment}
\usepackage{paralist}
\usepackage{enumitem}
\usepackage{amsmath}
\usepackage{algorithm}
\usepackage{algpseudocode}	
\usepackage{epstopdf}

\usepackage[numbered]{mcode}
\usepackage{listings}
\usepackage{color} 
\definecolor{mygreen}{RGB}{28,172,0} 
\definecolor{mylilas}{RGB}{170,55,241}

\newtheorem{theorem}{Theorem}

\newtheorem{lemma}[theorem]{Lemma}

\newtheorem{remark}[theorem]{Remark}

\begin{document}

\title{On the Complexity of the Weighted Fused Lasso}

\author{Jos\' e Bento\quad jose.bento@bc.edu
\;\;\;\;\;\;\;
Ralph Furmaniak \quad cf@acm.org 
\;\;\;\;\;\;\;   Surjyendu Ray \quad raysc@bc.edu
}

\markboth{On the Complexity of the Weighted Fused Lasso}{}

\maketitle

\begin{abstract}

The solution path of the 1D fused lasso for an $n$-dimensional
input is piecewise linear with $\mathcal{O}(n)$
segments \cite{hoefling2010path, tibshirani2011solution}. 
However, existing proofs of this
bound do not hold for the \emph{weighted} fused lasso.
At the same time, results for the generalized lasso, 
 of which the weighted fused lasso is a special case,
allow $\Omega(3^n)$ segments \cite{mairal2012complexity}. 
In this paper, we prove that the number of segments 
in the solution path of the weighted fused lasso is $\mathcal{O}(n^2)$, and that, for some instances, it is $\Omega(n^2)$.
We also give a new, very simple, proof
of the $\mathcal{O}(n)$ bound for the fused lasso.
\end{abstract}


%
\IEEEpeerreviewmaketitle

\section{Introduction}
\label{sec:intro}

The generalized lasso solves
\vspace{-0.2cm}
\begin{equation}\label{eq:gen_lasso}
\underset{x \in \mathbb{R}^n}{\text{minimize}}\;
 \frac{1}{2} \|y - A x\|^2_2 + \gamma \|Dx\|_1,
\end{equation}
%
%
where $\gamma \geq 0$, $D\in \mathbb{R}^{p \times n}$, $A \in \mathbb{R}^{m \times n}$ and $y \in \mathbb{R}^m$.
A special case of this problem, which is important for
signal processing (see \cite{chambolle2016geometric} and references therein for applications), is 
\vspace{-0.2cm}
\begin{equation}\label{eq:wei_fused_lasso}
\underset{x \in \mathbb{R}^n}{\text{minimize}}\; \frac{1}{2}\sum^{n}_{t=1} (x_t - y_t)^2+ \gamma \sum^{n-1}_{t=1} \alpha_t |x_{t+1}-x_t|.
\end{equation}
%
%
We call \eqref{eq:wei_fused_lasso} the \emph{weighted 1-D fused lasso} (W1FL) with input $y$ and weights $\alpha_t \geq 0\; \forall t$, and we distinguish it from the well studied special case when $\alpha_t =1 \; \forall t$, which is known as the 1-D fused lasso (1FL). 

There are efficient algorithms to solve 1FL for a fixed $\gamma$. \emph{Direct algorithms}, algorithms that solve a problem exactly in a finite number of steps, include the Taut String algorithm, \cite{davies2001local}, and the algorithm of \cite{johnson2013dynamic}, based on dynamic programing,
both with a worst case complexity of $\mathcal{O}(n)$;
and the algorithm of \cite{condat2013direct}, which is very fast in practice, but has $\mathcal{O}(n^2)$ worst case complexity. This algorithm has recently been improved to  finish in $\mathcal{O}(n)$ steps, \cite{condat2017fasterOnversion}. There are also algorithms that can deal with W1FL, for fixed $\gamma$, in $\mathcal{O}(n^2)$ iterations, \cite{barbero2014modular}, and in $\mathcal{O}(n)$ iterations, \cite{dumbgen2009extensions}. 
\emph{Iterative algorithms}, mostly first-order fixed-point
methods, include \cite{beck2009fast, chambolle2011first, wahlberg2012admm, bonettini2012convergence, ye2011split,barbero2011fast,condat2013primal, ramdas2016fast,pang2016primal}. Some of these
are based on the ADMM method, known to achieve the fastest possible
convergence rate among all first order methods,\cite{francca2016explicit,francca2017distributed}.
However, in many applications, when precision is crucial, or when implementing a termination procedure has a non-negligible computational cost, direct algorithm are preferred.

Frequently, we are not just interested in solving \eqref{eq:wei_fused_lasso} for a single $\gamma$.
 Let $x^*(\gamma)$ be the unique solution of \eqref{eq:wei_fused_lasso} \footnote{The generalized lasso \eqref{eq:gen_lasso} might not always have a unique solution \cite{tibshirani2013lasso}.}.
 An important problem is characterizing the set $\{x^*(\gamma):\gamma\geq0\}$, known as the \emph{solution path} of \eqref{eq:wei_fused_lasso}.
This might be necessary, for example, to efficiently ``tune''
the 1FL, i.e., find the value of $\gamma$ that gives best the result in a given application, \cite{dong2011automated, wen2012parameter}.

Another example is when we want to use a W1FL-path-solver to find the unique solution $\tilde{x}^*(\tilde{\gamma})$ of 
\vspace{-0.2cm}
\begin{equation}\label{eq:wei_fused_lasso_projection}
 \underset{\tilde{x} \in \mathbb{R}^n}{\text{min}}\frac{1}{2}\sum^{n}_{t=1} (\tilde{x}_t - y_t)^2  \text{ subject to } \sum^{n-1}_{t=1} \alpha_t |\tilde{x}_{t+1}-\tilde{x}_t| \leq \tilde{\gamma}.
\end{equation}
%
%
\vspace{-0.15cm}
One approach 
is to see, cf. \cite{loris2009performance}, that 
$\tilde{x}^*(\tilde{\gamma})=x^*({\gamma})$
if 
\begin{equation}\label{eq:projection_to_non_projection}
\gamma = \max_i |y_i - \tilde{x}_i^*(\tilde{\gamma})| \text{ or inversely if } 
\tilde{\gamma} = \|x^*(\gamma)\|_1.
\end{equation}

\vspace{-0.13cm}
We can then use the path $\{x^*(\gamma)\}$, and  \eqref{eq:projection_to_non_projection}, to find which $\gamma$ we should use in our solver to get $\tilde{x}^*(\tilde{\gamma})$.  Relations \eqref{eq:projection_to_non_projection} show that finding the solution paths of \eqref{eq:wei_fused_lasso_projection} and of \eqref{eq:wei_fused_lasso} is equivalent.

Characterizing $\{x^*(\gamma)\}$ is possible because $x^*(\gamma)$ is a continuous
piecewise linear function of $\gamma$, with a finite number $T$ of 
different linear segments, a result that follows directly from the KKT conditions, \cite{tibshirani2011solution}. Therefore, to characterize $x^*(\gamma)$, we only need to find the \emph{critical values} $\{\gamma_i\}^{T-1}_{i=1}$ at which $x^*(\gamma)$ changes linear segment, and the value of $x^*(\gamma)$ at these $\gamma$'s.

All efficient existing algorithms that find the solution path $\{x^*(\gamma)\}$ in a finite number of steps are essentially homotopy algorithms. These start with $x^*(\gamma_1=0) = y$, and sequentially compute $x^*(\gamma_{i+1})$ from $x^*(\gamma_{i})$.
One example is the algorithm of \cite{tibshirani2011solution}, that for 1FL has a complexity of $\mathcal{O}(n\log^2{n})$, with a special heap implementation. The best method is the primal path algorithm of \cite{hoefling2010path}, with $\mathcal{O}(n\log{n})$ complexity.

\begin{figure}[t]
 \begin{minipage}{.5\linewidth}
 \centering
  \includegraphics[width=4.cm,trim={0.cm 0 0.cm 0.cm},clip]{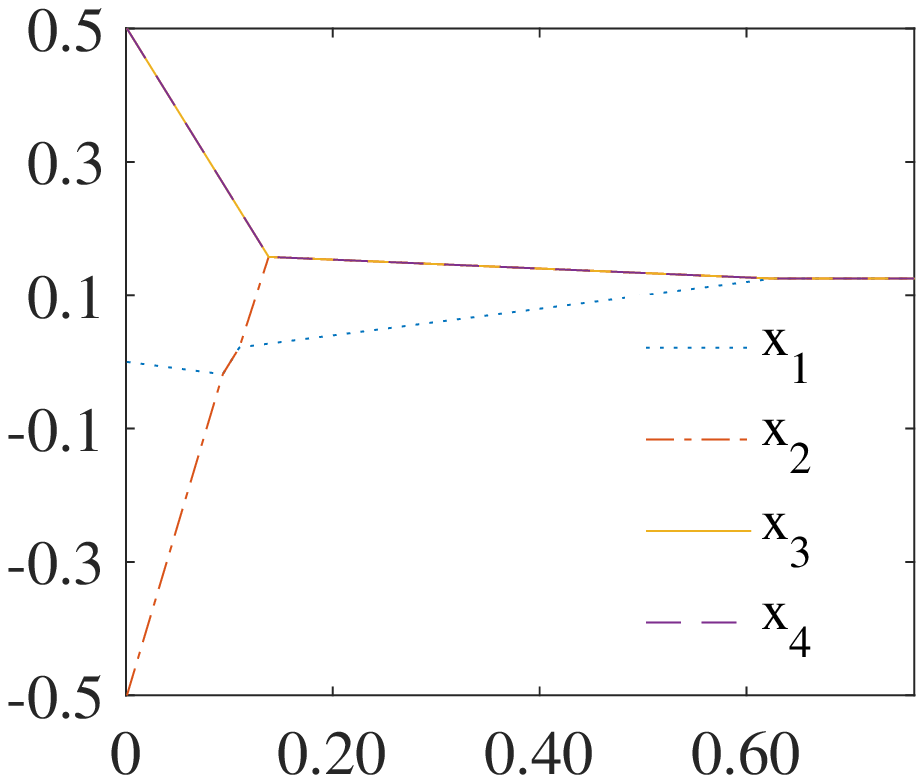}
  \put(-93,77){\small $x^*$ as a function of $\gamma$}
  \end{minipage}
  \begin{minipage}{.45\linewidth}\small
If $\alpha_1 = 1/50$, $\alpha_2 = 1/2$, $\alpha_3 = 1/2$, 
$y_1 = 0$, $y_2 = -1/2$, $y_3 = 1/2$ and $y_4 = 1/2$ then
$x^*_1 = x^*_2$ when $\gamma \in [\frac{25}{27},\frac{25}{23}] \cup [\frac{25}{4},\infty)$ and $x^*_1 \neq x^*_2$ otherwise. At $\gamma = \frac{25}{23}$, variables $x_1$ and $x_2$``un-fuse'' and at $\gamma \in \{ \frac{25}{27},\frac{25}{4}\}$ they ``fuse''.
    \end{minipage}
  \label{fig:example_of_fusing_and_unfusing}
  \caption{Example showing that in W1FL, variables can ``fuse'' and ``un-fuse''. Hence, existing proofs that $T = \mathcal{O}(n)$ do not hold for W1FL.}
\vspace{-0.6cm}
\end{figure}

To understand the complexity of finding $\{x^*(\gamma)\}$, we have to bound $T$. For 1FL,  \cite{friedman2007pathwise} proves that $T=\mathcal{O}(n)$, and \cite{tibshirani2011solution,tibshiranisupplement } give a  different proof of the same fact. These proofs' idea is as follows. The non-differentiable penalty $\sum_t |x_{t+1} - x_t|$ in \eqref{eq:wei_fused_lasso} implies that we have a critical point whenever, as $\gamma$ increases, for some $t$, the term 
$|x^*_{t+1} - x^*_t|$, goes from non-zero to zero, or vice-versa. When this happens, we say that $x^*_{t+1}$ and $x^*_t$ ``fuse'', or, conversely, that they ``un-fuse''. One then proves that, as $\gamma$ increases, variables never ``un-fuse''. Hence, there are at most $n$ fusing events, and thus $T \leq n$. 

Unfortunately, existing proofs do not extend to W1FS.
Figure \ref{fig:example_of_fusing_and_unfusing} is an example where variables ``fuse'' and ``un-fuse''. Hence, to bound $T$, we are left with bounds for the generalized lasso which, in a worst case scenario, can be $\Omega(3^n)$ \cite{mairal2012complexity}.

{$\bullet$} Our main contribution is to show that $T = \mathcal{O}(n^2)$, and that, in a worst case scenario, $T = \Omega(n^2)$.
%

%
%
%
\vspace{-0.2cm}
\section{Main results}

We start by reformulating W1FS, as stated in Theorem \ref{th:spring_representation}. In this theorem, and throughout the paper, we use the notation $[n] = \{1,\dots,n\}$.

\begin{theorem}\label{th:spring_representation}
Let $\tilde{y}_i = -\sum^n_{t=i} y_t$ for $i \in [n+1]$, where we assume that $\tilde{y}_{n+1} = 0$. Let $\tilde{\alpha}_{1} = \tilde{\alpha}_{n+1} = 0$, and
$\tilde{\alpha}_{i+1} = \alpha_{i}$ if $i \in [n-1]$.
Let $w^*(\gamma)$ be the unique minimizer of 
\begin{equation}\label{eq:spring_interpretation}
\min_{w\in \mathbb{R}^n} \sum^{n}_{i = 1} (w_{i+1}-w_{i})^2 \text{ s.t. } |w_i - \tilde{y}_i| \leq \gamma \tilde{\alpha}_{i}, i \in [n+1]
\end{equation}
We have that $x^*(\gamma)_t = w^*(\gamma)_{t+1}  - w^*(\gamma)_{t}$ for all $t \in [n]$.
\end{theorem}
See \cite{barbero2014modular}, e.g., for a proof of Theorem \ref{th:spring_representation}. The Supplementary Material
includes another proof, based on the Moreau identity.

Theorem \ref{th:spring_representation} allows us to study
the number linear segments in $x^*$ by studying the number linear segments in $w^*$. Indeed, since by Theorem \ref{th:spring_representation} we have that $x^*(\gamma)_t = w^*(\gamma)_{t+1}  - w^*(\gamma)_{t}$, it follows that $x^*(\gamma)$ only changes linear segment
if $w^*(\gamma)$ changes linear segment. Hence, if there are at most $T$ different linear segments in $w^*(\gamma)$, there are at most $T$ different linear segments in $x^*(\gamma)$.
Similarly, if there are at least $T$ linear segments  in $w^*(\gamma)$, and if, for each $\gamma$, no two consecutive components of $w^*$ change linear segment, then there are at least $T$ linear segments in $x^*(\gamma)$.

We can make a few
simple observations about $w^*(\gamma)$, defined in Theorem  \ref{th:spring_representation},
which we use to get bounds on $T$.
%
First we introduce some notation.
Figure \ref{fig:illustration_of_notation} illustrates its use.
\begin{itemize}[leftmargin=*]
\item We refer to $w^*_i$ as the $i$th variable or $i$th \emph{point}.
We refer to $[\tilde{y}_i - \tilde{\alpha}_i\gamma,\tilde{y}_i + \tilde{\alpha}_i\gamma]$ as the \emph{interval associated to} the $i$th point. We say that the $i$th point is \emph{touching its left or right boundary} if $w^*_i = \tilde{y}_i - \tilde{\alpha}_i\gamma$ or if $w^*_i = \tilde{y}_i + \tilde{\alpha}_i\gamma$ respectively.  A point that is not touching either side of the boundary of its interval is called \emph{free}. Otherwise, it is called \emph{non-free}.
\item We define $F(\gamma) = \{i: |w^*(\gamma)_i  - \tilde{y}_i| < \tilde{\alpha}_i\gamma\}$  and $B(\gamma) = [n+1]\backslash F$. In words, a point is in $F$ if and only if it is it is free. A point is in $B$ if and only if it is not free.
\item  We define $s_i(\gamma) = 1$ if and only if $w^*(\gamma)_i  = \tilde{y}_i + \tilde{\alpha}_i\gamma$ and $s_i(\gamma) = -1$ if and only if $w^*(\gamma)_i  = \tilde{y}_i - \tilde{\alpha}_i\gamma$.
For $i = 1$ and $i = n+1$, for which the left and right boundaries are the same, we choose $s_i$ by convention. It does not matter which values we choose.
\item For the $i$th point, we define $i^\triangleleft(\gamma) =  \max \{j \in B(\gamma): j < i\}$
and let $i^\triangleright(\gamma) =  \min \{j \in B(\gamma): i < j\}$.
In words, $i^\triangleleft$ and $i^\triangleright$ are the pair of indices 
of non-free points, smaller and larger than $i$ respectively, that are closer to $i$.
\end{itemize}
For simplicity, and whenever clear from the context, we omit the dependency in $\gamma$ in our expressions. We now list our observations.
\begin{figure}[h!]
\centering
    \includegraphics[width=8.35cm,trim={0.cm 0 0.cm 0.cm},clip]{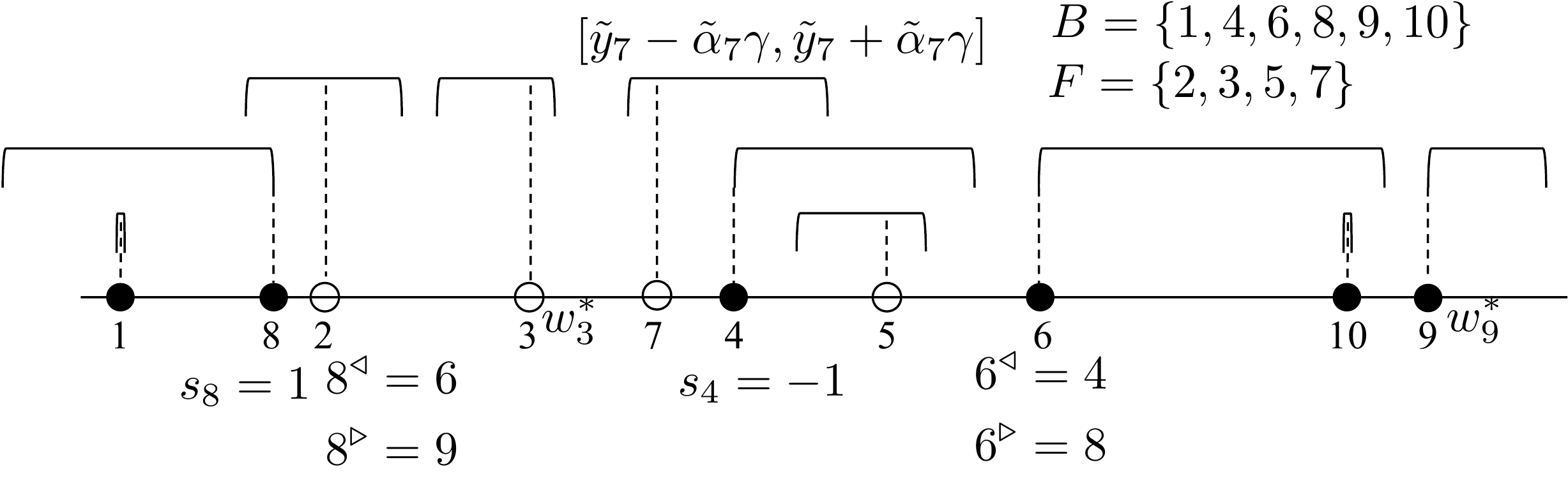}
    \caption{Illustration of the notation used. Square horizontal brackets represent intervals inside which each variable must be. Solid circles represent variables touching their intervals' boundary. In this picture $n = 9$, so $\tilde{\alpha}_1=\tilde{\alpha}_{n+1}=0$.}
  \label{fig:illustration_of_notation}
\end{figure}
%

{\bf 1)} Since $w^*(\gamma)$ is a continuous function
of $\gamma$ (see e.g. \cite{tibshirani2011solution}), if a point $i$ has $\tilde{\alpha}_i > 0$, then
it cannot touch its left (right) boundary for $\gamma_2$ and touch its right (left) boundary for $0 < \gamma_1 < \gamma_2$ without being free for at least one $\gamma \in (\gamma_1,\gamma_2)$. This implies that if $B$ and $F$ do not change for all $\gamma \in [\gamma_1, \gamma_2]$, then $s_i$ is constant in $\gamma \in [\gamma_1, \gamma_2]$.

{\bf 2)} If $i \in F$, then 
$w^*_i$ is dictated by the position of the $i^\triangleleft$th and $i^\triangleright$th points, namely,
\begin{align}\label{eq:movement_dictated_by_extremes}
w^*_i &= w^*_{i^\triangleleft} f_i+ w^*_{i^\triangleright} \bar{f}_i,
\end{align}
where $f_i =  \frac{i^\triangleright - i}{i^\triangleright - i^\triangleleft}$ and $\bar{f}_i = 1 - f_i$.
Equation \eqref{eq:movement_dictated_by_extremes} follows from the fact that, for any $i^\triangleleft< s <  i^\triangleright$, $w^*_s$ comes from the solution of the problem $\min_{\{w_s\}} \sum^{i^\triangleright-1}_{s = i^\triangleleft} (w_{s+1}-w_s)^2$ subject to $w_{i^\triangleleft} = w^*_{i^\triangleleft}$ and $w_{i^\triangleright} = w^*_{i^\triangleright}$. Its solution is that the points $\{w^*_s\}$ must divide the interval $[w^*_{i^\triangleleft},w^*_{i^\triangleright}]$
in $i^\triangleright - i^\triangleleft$ equal parts, hence \eqref{eq:movement_dictated_by_extremes}.
Note that, if for all $\gamma \in [\gamma_1, \gamma_2]$,
$F$ does not change then, by observation 1, $i^\triangleright$, $i^\triangleleft$, $s_{i^\triangleright}$
and $s_{i^\triangleleft}$ are constant for all $\gamma \in [\gamma_1, \gamma_2]$ and \eqref{eq:movement_dictated_by_extremes} becomes
a linear function of $\gamma$ given by
\begin{align}\label{eq:movement_dictated_by_extremes_2}
w^*_i(\gamma) &= (\tilde{y}_{i^\triangleleft} f_i + \tilde{y}_{i^\triangleright} \bar{f}_i) + \gamma
(\tilde{\alpha}_{i^\triangleleft }s_{i^\triangleleft} f_i + \tilde{\alpha}_{i^\triangleright} s_{i^\triangleright}\bar{f}_i).
\end{align}
This also implies that a change in linear segment, and hence a critical point, only occurs when $F$ and $B$ change.

{\bf 3)} We give a necessary condition for $\gamma_c$ to be a critical point,
at which $w^*_i$
transitions from being free to non-free as we increase $\gamma$.
Since $w^*$ is continuous and piecewise linear with a finite number of linear segments (see \cite{tibshirani2011solution}),
we know that $F$ is constant in a small enough interval
of the form $I = (\gamma',\gamma_c)$.
We can then use \eqref{eq:movement_dictated_by_extremes_2} 
for all $\gamma \in I$ and the continuity property to conclude that $\gamma_c$
must satisfy
\begin{equation}\label{eq:critical_entering_time}
(\tilde{y}_{i^\triangleleft} f_i + \tilde{y}_{i^\triangleright} \bar{f}_i) + \gamma_c
(\tilde{\alpha}_{i^\triangleleft }s_{i^\triangleleft} f_i + \tilde{\alpha}_{i^\triangleright} s_{i^\triangleright}\bar{f}_i) = \tilde{y}_i + \tilde{\alpha}_i s_i \gamma_c,
\end{equation}
where $s_i$ is evaluated at $\gamma = \gamma_c$
and $i^\triangleleft$, $i^\triangleright$, $s_{i^\triangleleft}$ and $s_{i^\triangleright}$
are evaluated at any point in $I$.

Furthermore, assume that $s_i(\gamma_c) = +1$, then,
for $\gamma \in I$, the left hand side, l.h.s., of \eqref{eq:critical_entering_time} is strictly smaller than the right hand side, r.h.s, and, as $\gamma$ increases to $\gamma_c$, the l.h.s must increase until it is equal to the r.h.s. Hence, we conclude that the following relation holds between their rates of growth
$\tilde{\alpha}_{i^\triangleleft }s_{i^\triangleleft} f_i + \tilde{\alpha}_{i^\triangleright} s_{i^\triangleright}\bar{f}_i >  \tilde{\alpha}_i s_i. 
$
 Similarly, if $s_i(\gamma_c) = -1$, then 
$
\tilde{\alpha}_{i^\triangleleft }s_{i^\triangleleft} f_i + \tilde{\alpha}_{i^\triangleright} s_{i^\triangleright}\bar{f}_i <  \tilde{\alpha}_i s_i. 
$
We can thus write that, if the $i$th point transitions from
free to non-free, then
\begin{equation} \label{eq:free_to_non_free}
\tilde{\alpha}_{i^\triangleleft }s_is_{i^\triangleleft} f_i + \tilde{\alpha}_{i^\triangleright} s_is_{i^\triangleright}\bar{f}_i >  \tilde{\alpha}_i . 
\end{equation}

$\bullet$ Our first result is a new very simple proof, when compared to \cite{friedman2007pathwise,tibshirani2011solution,tibshiranisupplement }, of the known fact that, for 1FL,  $T = \mathcal{O}(n)$. Our second and third result are new altogether.

\begin{theorem}\label{th:linear_complexity_1FL}
1FL has at most $\mathcal{O}(n)$ different linear segments.
\end{theorem}
\begin{proof}
It is enough to prove that $w^*$ has at most $\mathcal{O}(n)$ different linear segments.
For 1FL, we have $\tilde{\alpha}_{i+1} = 1$ for all $i \in [n-1]$. Therefore,
if $w^*_i$ is free and, as $\gamma$ increases, it becomes non-free, then, by \eqref{eq:free_to_non_free}, we must have that
$
1 < s_is_{i^\triangleleft} f_i +  s_is_{i^\triangleright}\bar{f}_i.
$
This implies that 
$
1 < |s_is_{i^\triangleleft} f_i +  s_is_{i^\triangleright}\bar{f}_i| \leq  f_i +  \bar{f}_i = 1,
$
which is a contradiction. Therefore, a point $i$ never goes from $F$ to $B$. Since a critical point in $w^*(\gamma)$ only occurs when $F$ changes (by observation 2), and since $F$ can only change by the addition of $n-1$ variables at most, we have $T = \mathcal{O}(n)$.
\end{proof}

\begin{theorem}
W1FL has $\mathcal{O}(n^2)$ different linear segments.
\end{theorem}
\begin{proof}
At any critical value $\gamma_c$ at which $B$ changes, $B$ might change by multiple points. Some added others removed. Let $i(\gamma_c,r)$ be the $r$th element 
changing in $B$ at $\gamma_c$, and let $i^\triangleright(\gamma_c,r)$ and $i^\triangleleft(\gamma_c,r)$ be indices of non-free points, larger and smaller than $i(\gamma_c,r)$ respectively, that are closer to $i(\gamma_c,r)$.

Let $i < j$ and consider the set
$
S_{ij} = \{\gamma_c:(i^\triangleleft(\gamma_c,r),i^\triangleright(\gamma_c,r)) = (i,j)\text{ for some } r\},
$
 i.e., the set of critical values at which some point that is being added or removed from $B$ has as its closest non-free points to the left and to the right the points $i$ and $j$. We claim, and latter prove, that $|S_{ij}| \leq 8$. The total number of critical points is equal to $\sum_{i<j} |S_{ij}| \leq 8 {n \choose 2} = \mathcal{O}(n^2)$, and we are done.

Now we prove that $|S_{ij}| \leq 8$. Let $\gamma_c  \in S_{ij}$. Let $i<k < j$. By equation \eqref{eq:movement_dictated_by_extremes}, we know that $w_k^*(\gamma_c)$ can be expressed as a linear function of $\gamma_c$. Hence, together with the fact that we must have
$w_k^* \in 
\tilde{y}_k + \gamma_c[-\tilde{\alpha}_k,\tilde{\alpha}_k]$, we know that $\gamma_c \in [a_k,b_k]$, for some $a_k \geq 0$ and $b_k$, which depend
on $s_i$ and $ s_j$, and the parameters $\tilde{\alpha},\tilde{y}$ of the problem.  $b_k$ might be $\infty$. 
In words, for $w^*_k(\gamma_c)$ to be feasible, $\gamma_c$ must be in some interval $ [a_k,b_k]$

This implies that $\gamma_c\in \cap_{i < k < j} [a_k,b_k] = [A_{ij},B_{ij}]$, for some $A_{i,j} \geq 0$ and $B_{i,j}$, which depend
on $s_i$ and $ s_j$, and the parameters $\tilde{\alpha},\tilde{y}$ of the problem. $B_{i,j}$ might be $\infty$.

We argue that it must be that $\gamma_c$ is equal to either $A_{ij}$ or $B_{ij}$. Indeed, since $\gamma_c$ is a critical point, if $k = i(\gamma_c,r)$ for some $r$, then  $w^*_k$ must be at the boundary of $\tilde{y}_k + \gamma_c[-\tilde{\alpha}_k,\tilde{\alpha}_k]$. This implies that $\gamma_c$ is either $a_k$ or $b_k$, and hence that $\gamma_c$ is at the boundary of $\cap_{i < k' < j} [a_{k'},b_{k'}]$.
There are $4$ choices for the pair $(s_i,s_j)$, and, for each of these choices, $\gamma_c$ can be either $A_{ij}$ or $B_{ij}$. Hence, for any pair $(i,j)$, there are at most $8$ possible values for $\gamma_c$. Thus $|S_{ij}| \leq 8$.
\end{proof}

\begin{theorem}\label{th:lower_bound}
There exists $\alpha$ and $y$ such that W1FL
has $\Omega(n^2)$ different linear segments. One example is to chose $\alpha$ and $y$ such that $\tilde{\alpha}$ and $\tilde{y}$ satisfy
\begin{align}
&\tilde{\alpha}_i = (i-1)^2, \forall i\in[n],\tilde{\alpha}_{n+1}=0,\label{eq:counter_1}\\
&\tilde{y}_i = (-1)^i q_i, \forall i\in[n],\tilde{y}_{n+1}=0,\\
&q_{1} \hspace{-1mm} =\hspace{-0.8mm}1,q_{2} \hspace{-1mm} =\hspace{-0.8mm}2,q_{i+2}\hspace{-1mm} =\hspace{-0.8mm}2q_{i+1}\hspace{-1mm} -\hspace{-0.8mm}q_{i} \hspace{-1mm} +\hspace{-0.8mm} 2 g_{i+2} \hspace{-1mm} +\hspace{-0.8mm}1, \forall i\in[n-2],\\
&g_3 = 1/3, g_{i+3} = 2g_{i+2} +1 \label{eq:counter_4}.
\end{align}
\end{theorem}
\begin{remark}
This theorem automatically implies that there are examples for which variables ``fuse'' $\Omega(n^2)$ times and ``un-fuse''
$\Omega(n^2)$ times. Furthermore, its proof implies that the different between the number of ``fuse'' and ``un-fuse'' events is $\mathcal{O}(n)$.
\end{remark}

What is the idea behind Theorem \eqref{th:lower_bound}? The fact $\tilde{\alpha}_i$ grows super-linearly with $i$, allows us to have a value of $\gamma$ around which the $i$th interval is responsible for driving the behavior of $w^*(\gamma)$. 
Let us call this the \emph{$i$th epoch}. 
The fact that $\tilde{y}_i$ oscillates and diverges exponentially fast with $i$, drives points, between epochs, to alternate between being non-free at their right or left boundary, according to whether $\tilde{y}_i$ is very large and negative or very large and positive. Since there are $n$ epochs, and since, in general, from one epoch to the next, $\Omega(n)$ points change from their right to left boundary (or vice versa), we have $\Omega(n^2)$ ``fuse'' and ''un-fuse'' events.

\begin{proof}
We are going to produce a set of sufficient conditions
that guarantee that the number of critical points in $w^*$ 
is $\Omega(n^2)$. It is then an algebra exercise, which we omit, to check that these conditions are satisfied by our choice above.
Finally, we prove why these same conditions imply
that no two consecutive components of $w^*$ change their linear segment at the same time, and hence why $x^*$ also
has at least $\Omega(n^2)$  different linear segments.

In particular, our conditions will imply the existence of $\tilde{\alpha}_2,\dots, \tilde{\alpha}_{n}\neq 0$ and $\tilde{y}_1,\dots\tilde{y}_n$, such that there exits a sequence of critical points $0 < \gamma_3 < \gamma_4<\dots<\gamma_{n}$ such that
 the following two scenarios hold true:
\begin{enumerate}[leftmargin=*]
\item for $\gamma = \gamma_{2k}$, $3 \leq 2k \leq n$, every point $w^*_i$ with $i \in \{2,\dots,2k-1\}$ is touching its right boundary;\label{enu:scenario1}
\item  for $\gamma = \gamma_{2k+1}$, $3 \leq 2k+1 \leq n$, every point  $w^*_i$ with $i \in \{2,\dots,2k\}$ is touching its left boundary;\label{enu:scenario2}
\end{enumerate}
Both scenarios imply that, for $\gamma \in [\gamma_r, \gamma_{r+1})$, $r \in \{3,\dots,n-1\}$,
every point $w^*_i$, $i\in \{2,\dots,r-1\}$, changes from touching
one side of its boundary to the other side of its boundary.
Note that since $\tilde{\alpha}_2,\dots\tilde{\alpha}_{n}\neq 0$, if $\gamma > 0$, a point cannot be simultaneously touching its left and right boundary.
Hence, $w^*(\gamma)$ has at least $r-2$ critical points in $\gamma \in [\gamma_r, \gamma_{r+1})$. Hence, for $\gamma \in [\gamma_2,\gamma_n]$, there
are at least 
\begin{equation}\label{eq:critical_point_estimate}
1 + \dots + n-1 = n (n-1)/2
\end{equation}
critical points.

Let us be more specific about the two scenarios.
In Scenario \ref{enu:scenario1}, in addition to what we have already described,
for $\gamma = \gamma_{r}$ and $r=2k$, we also want the intervals $[\tilde{y}_1 - \tilde{\alpha}_1 \gamma,\tilde{y}_1 + \tilde{\alpha}_1\gamma]$, \dots,
$[\tilde{y}_{r-1} - \tilde{\alpha}_{r-1}  \gamma,\tilde{y}_{r-1}  + \tilde{\alpha}_{r-1} \gamma]$ to be nested, with the intervals for large $i$ containing
the intervals for small $i$. We also want the left boundary of the interval
$[\tilde{y}_{r} - \tilde{\alpha}_{r} \gamma,\tilde{y}_{r} + \tilde{\alpha}_{r}\gamma]$
to be larger than the right most limit of all the intervals involving indices smaller than $r$.
The following picture illustrates these conditions.
\includegraphics[width=8.35cm,trim={0.cm -0.5cm 0.cm -0.5cm},clip]{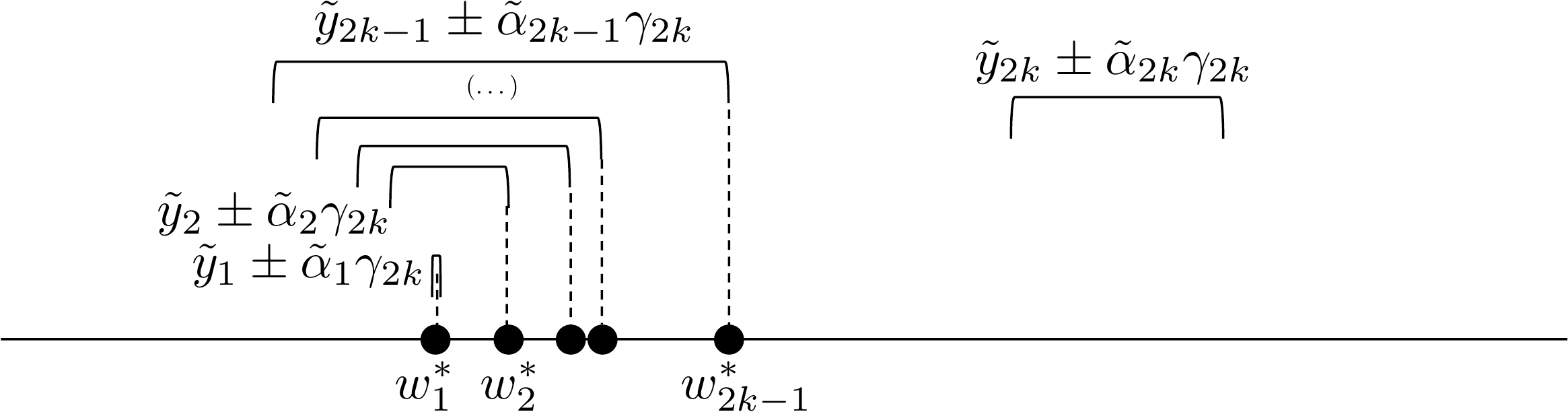}

In Scenario \ref{enu:scenario2}, in addition to what we have already described,
for $\gamma = \gamma_{r}$ and $r=2k+1$, we want the same conditions
as in  Scenario \ref{enu:scenario1} but now with left and right reversed. 
The following picture illustrates these conditions.
\includegraphics[width=8.35cm,trim={0.cm -0.5cm 0.cm -0.5cm},clip]{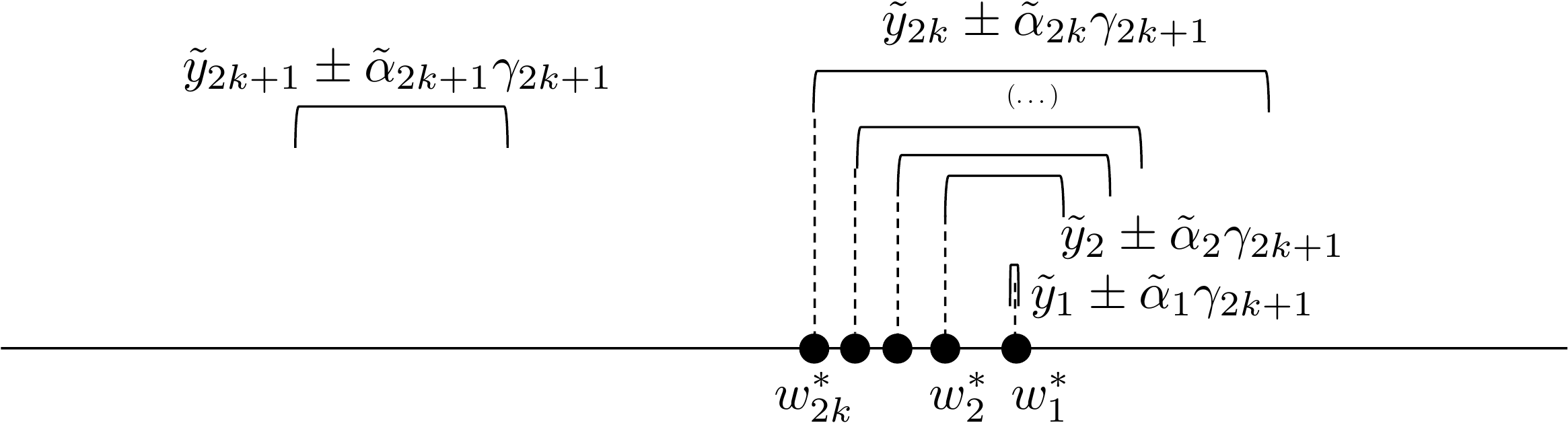}
Note that in both scenarios, we do not care where the points $r+1,\dots,n$ are.

As we explain next, a set of \emph{sufficient} conditions for Scenario \ref{enu:scenario1} (Scenario \ref{enu:scenario2}) to hold is that, for all $r\in\{3,\dots,n\}$,
\begin{align}
&| \tilde{y}_{i+1} - \tilde{y}_{i}|< (\tilde{\alpha}_{i+1} -\tilde{\alpha}_{i})  \gamma_r, \forall i \in [r-2],\label{eq:cond1_sce_1}\\
&|\tilde{y}_{i+2}\hspace{-1mm} -\hspace{-1mm}2\tilde{y}_{i+1} \hspace{-1mm} +\hspace{-0.8mm}\tilde{y}_{i} | \hspace{-1mm} <\hspace{-1mm} (\tilde{\alpha}_{i+2} \hspace{-1mm} -\hspace{-1mm} 2\tilde{\alpha}_{i+1} \hspace{-1mm} +\hspace{-0.8mm} \tilde{\alpha}_{i}) \gamma_r , \forall i \in[r-3],\label{eq:cond2_sce_1}\\
&(-1)^{r}(\tilde{y}_r- 2\tilde{y}_{r-1} +\tilde{y}_{r-2}) \hspace{-1mm} >\hspace{-1mm} (\tilde{\alpha}_r + 2\tilde{\alpha}_{r-1} - \tilde{\alpha}_{r-2}) \gamma_r \label{eq:cond3_sce_1}.
\end{align}
Condition \eqref{eq:cond1_sce_1} directly implies that the first $r-1$ intervals are nested.
Condition \eqref{eq:cond2_sce_1} implies that if the $(i+2)$th and $i$th points
are touching their right (left) boundary, then the $(i+1)$th point must be touching
its right (left) boundary. This can be seen by solving the simple quadratic problem $\min_w (w_{i+2}-w_{i+1})^2 + (w_{i+1}-w_{i})^2$ subject to $w_{i+2}$ and $w_{i+1}$ behind at their boundaries, and $w_{i+1}$ being inside its interval.
%
Condition \eqref{eq:cond3_sce_1} implies that, no matter where $w^*_r$ is, even if $w^*_r$ is touching its left (right) boundary, the $(r-1)$th point must touch its right (left) boundary.
This can also be seen by solving a simple quadratic problem involving three variables.
These three conditions together imply that the first $r-1$ points are touching their right (left) boundary.

As we explain next, these conditions are in turn implied by the following set of conditions. These also imply that $\gamma_r$ is increasing.
\begin{align}
&\tilde{\alpha}_{i+2} - 2\tilde{\alpha}_{i+1} +\tilde{\alpha}_{i} > 0,\label{eq:cond1_simpler}\forall i \in [n-2] ,\\
&\tilde{\alpha}_{i+1} - \tilde{\alpha}_{i}  > 0,\forall i \in [n-1],\label{eq:cond2_simpler}\\
&\gamma_3 > \frac{q_2-q_1}{\tilde{\alpha}_2 - \tilde{\alpha}_1},\label{eq:cond3.0_simpler}\\
&\gamma_{i+3} \hspace{-1mm} >\hspace{-1mm}  \max \left \{\gamma_{i+2},\frac{{q}_{i+2} \hspace{-1mm} -\hspace{-1mm}  2{q}_{i+1}\hspace{-1mm} +\hspace{-0.8mm} {q}_{i}}{\tilde{\alpha}_{i+2} \hspace{-1mm} -\hspace{-0.8mm}  2\tilde{\alpha}_{i+1} \hspace{-1mm} +\hspace{-0.8mm}  \tilde{\alpha}_{i}},\frac{q_{i+2}\hspace{-1mm} -\hspace{-0.8mm} q_{i+1}}{\tilde{\alpha}_{i+2} \hspace{-1mm} -\hspace{-0.8mm}  \tilde{\alpha}_{i+1}}\right\},\label{eq:cond3_simpler}\\
&  \hspace{6cm }\forall i \in [n-3],\nonumber \\
&q_2 > q_1\label{eq:cond4.0_simpler},\\
&{q}_{i+2}  \hspace{-1mm} >\hspace{-0.8mm} 2q_{i+1} \hspace{-1mm} -\hspace{-0.8mm}q_{i} \hspace{-1mm} +\hspace{-0.8mm} (\tilde{\alpha}_{i+2} \hspace{-1mm} +\hspace{-0.8mm} 2\tilde{\alpha}_{i+1} \hspace{-1mm} -\hspace{-0.8mm} \tilde{\alpha}_{i})\gamma_{i+2},\forall i \in [n-2], \label{eq:cond4_simpler}\\
&\tilde{y}_i = (-1)^i q_i, \forall i \in [n] \label{eq:cond5_simpler}.
\end{align}

Conditions \eqref{eq:cond2_simpler}, \eqref{eq:cond3.0_simpler} and
\eqref{eq:cond4.0_simpler} imply that $\gamma_3 > 0$.
Condition \eqref{eq:cond3_simpler} further implies that $ 0 < \gamma_3 < \dots <
\gamma_n$.
Condition \eqref{eq:cond4_simpler} and \eqref{eq:cond5_simpler} imply 
condition \eqref{eq:cond3_sce_1}. Since $\tilde{\alpha}>0$, condition \eqref{eq:cond1_simpler} implies that $\tilde{\alpha}_{i+2} + 2\tilde{\alpha}_{i+1} - \tilde{\alpha}_{i} > 0$ for all $i \in [n-2]$ and thus, since $\gamma_{i+2} >0$, condition \eqref{eq:cond4.0_simpler} together with \eqref{eq:cond4_simpler} imply that ${q}_{i+2} - 2q_{i+1} +  q_{i} >0$ for all $i\in [n-2]$ and that
${q}_{i+1} - q_{i} >0$ for all $i\in [n-1]$. 
Therefore,
$|\tilde{y}_{i+2} -2\tilde{y}_{i+1} +\tilde{y}_{i} | =  {q}_{i+2} - 2q_{i+1} +  q_{i}$ for all $i\in [n-2]$
and $|\tilde{y}_{i+1} -\tilde{y}_{i}  | =  {q}_{i+1} - q_{i}$
  for all $i\in [n-1]$.
  These two equations, together with conditions 
\eqref{eq:cond1_simpler}, \eqref{eq:cond2_simpler},
\eqref{eq:cond3.0_simpler} and \eqref{eq:cond3_simpler} imply that \eqref{eq:cond1_sce_1}
holds for all $i \in [n-1]$, if we replace $\gamma_r$ by
$\gamma_{i+2}$. These also imply that
 \eqref{eq:cond2_sce_1} holds for all $i \in [n-2]$, if we replace $\gamma_r$ by $\gamma_{i+3}$ there. But since $\gamma_i$
is increasing, we have that both \eqref{eq:cond1_sce_1} and \eqref{eq:cond2_sce_1} hold true exactly as specified.

We finally argue by contradiction why conditions
\eqref{eq:cond1_simpler}-\eqref{eq:cond5_simpler}
 also imply that no two consecutive components of $w^*$ stop touching, or start touching, their boundary at the same $\gamma$. We prove this only for values of $\gamma$ that are in some of the intervals $[\gamma_{r},\gamma_{r+1})$ that contribute to our $n^2$ estimate in \eqref{eq:critical_point_estimate}.

Assume that both $w^*_i$ and 
$w^*_{i+1}$ have a critical point at $\gamma = \gamma_c$.
Assume also that 
$\gamma_c \in [\gamma_{r},\gamma_{r+1})$ for some $r>i+1$. Note that $\gamma_c$ must satisfy this condition if it is to contribute for our $n^2$ estimate in \eqref{eq:critical_point_estimate}.
It must be that $|w^*_{i-1} - w^*_i| = | w^*_i - w^*_{i+1}| = |w^*_{i+1} -  w^*_{i+2}|$. At the same time, and as we saw in the previous paragraph, conditions
\eqref{eq:cond1_simpler}-\eqref{eq:cond5_simpler} imply 
that $|\tilde{y}_{i+1}  - 2\tilde{y}_{i}  + \tilde{y}_{i-1} |  <  (\tilde{\alpha}_{i+1}   -  2\tilde{\alpha}_{i}   + \tilde{\alpha}_{i-1}) \gamma_{r} \leq (\tilde{\alpha}_{i+1}   -  2\tilde{\alpha}_{i}   + \tilde{\alpha}_{i-1}) \gamma_{c}$, which in turn implies that $|w^*_i - w^*_{i-1}| < |w^*_{i+1} - w^*_{i}|$, which is a contradiction.
\end{proof}

\section{Numerical experiments}

Figure \ref{fig:numerical_break_points}-(left) shows a numerical computation of the number of critical points as a function of $n$ for the example \eqref{eq:counter_1}-\eqref{eq:counter_4}. As Theorem \ref{th:lower_bound} predicts, the
number of critical points grows quadratically with $n$. One can also observe that difference between ``fuse'' and ``un-fuse'' events is $\mathcal{O}(n)$.
For Figure \ref{fig:numerical_break_points}-(right), we generated $100$ random sets of $\alpha$ and $y$, and, for each size $n$, we show on the $y$-axis the average number of ``fuse'' events and ``un-fuse'' events observed over these $100$ runs. Each $\alpha_i$ was sampled from a uniform distribution in $[0,1]$, independently across $\alpha$'s, and each $y_i$ was sampled from a $\mathcal{N}(0,\sqrt{10})$, independently across $y$'s. Although Theorem \ref{th:lower_bound} tells us that we can observe $\Omega(n^2)$ ``fuse'' and ``un-fuse'' events, in our random instances for W1FL, ``un-fuse'' events are rare and both types of events seem to grow linearly with $n$.

Critical points can be computed using, for example,
\cite{tibshirani2011solution}. In  Supplementary Material we give simple algorithm which we use to compute the critical points  based on Theorem \ref{th:spring_representation}.

\begin{figure}[h!]
\includegraphics[width=5cm,trim={0cm 0cm 0cm 0cm},clip]{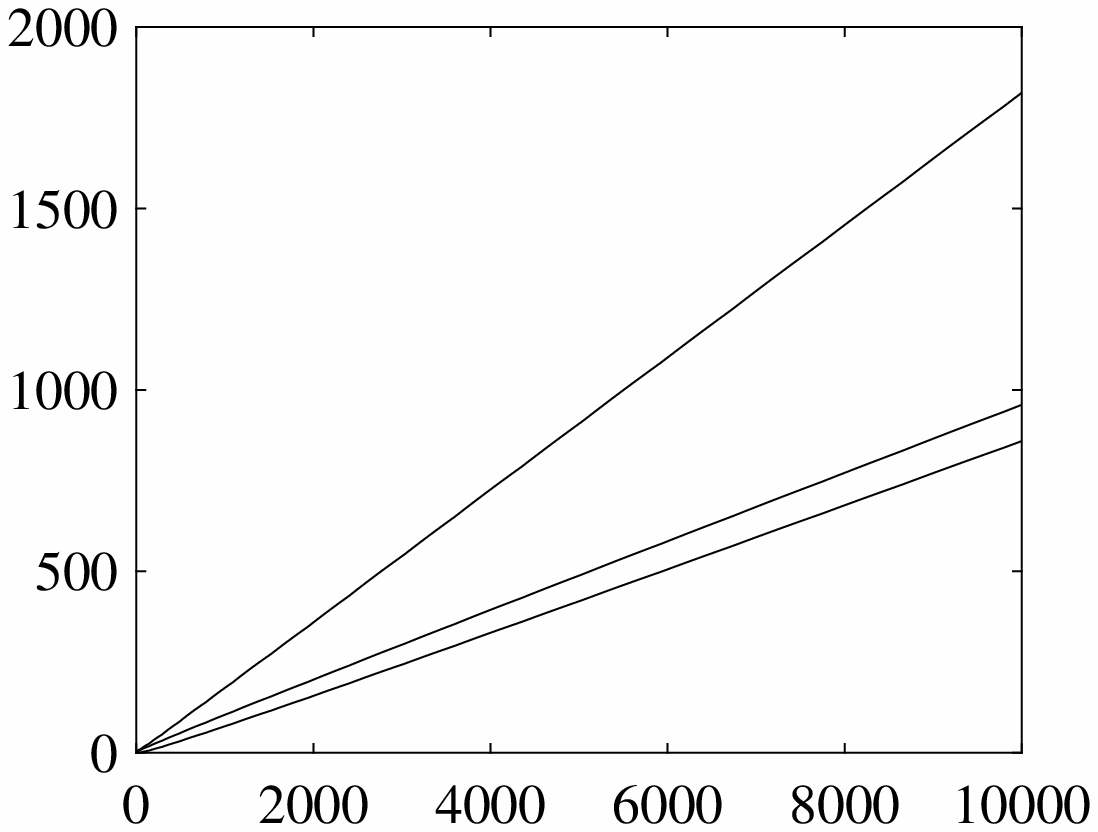}
        \put(-147,30){\rotatebox{90}{\footnotesize \# critical points}}
      \put(-90,-7){\footnotesize Input size squared, $n^2$}
        \put(-77,60){\footnotesize all events}
        \put(-77,40){\footnotesize ``fuse'' events}
                \put(-77,20){\footnotesize ``un-fuse'' events}
    \includegraphics[width=4.37cm,trim={0cm 0.cm 0.cm 0.0cm},clip]{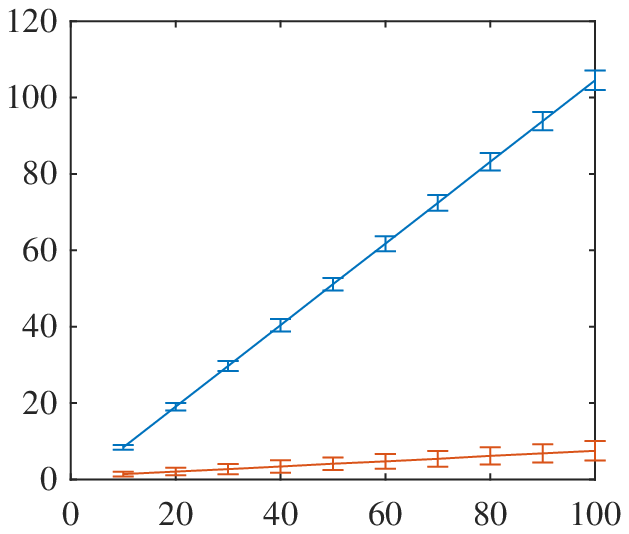}
      \put(-130,30){\rotatebox{90}{\footnotesize \# critical points}}
      \put(-75,-7){\footnotesize Input size, $n$}
      \put(-77,15){\footnotesize ``un-fuse'' events}
      \put(-77,65){\footnotesize ``fuse'' events}
    \vspace{-0.cm}
    \caption{(Left) Number of critical points as function of $n^2$ for the example of Theorem \ref{th:lower_bound}; (Right) Number of ``fuse'' and ``un-fuse'' events in random instances of W1FL, as a function of $n$.}
  \label{fig:numerical_break_points}
\end{figure}
%

\section{Future work}

The weighted fused lasso on an input of size $n$ is substantially different from
the equal-weights fused lasso: two consecutive components can both become equal (``fuse'') and become unequal (``un-fuse'') multiple times along the solution path. We have shown that there are instances with $\Omega(n^2)$ of these ``fuse''/``un-fuse'' events, and that no instance can have more than $\mathcal{O}(n^2)$ events. We have also produced a very simple proof of why, in the equal-weights fused lasso, there are $\mathcal{O}(n)$ events.

Future work should include finding conditions for the weights $\alpha$ and input $y$ under which (a) a $\mathcal{O}(n)$ bound holds and, (b) the number of ``fuse'' and ``un-fuse'' events are substantially different. It would then be useful to compute how likely it is for these conditions to be satisfied, under different stochastic models for the input. 

\section*{Acknowledgment}
We would like to thank Piotr Suwara for his important help
in some of the proofs. This work was partially supported by the grants NSF-IIS-1741129 and NIH-1U01AI124302.

\bibliography{biblio}
\bibliographystyle{ieeetr}

%
%
%
\newpage
\clearpage

{\bf  \large Supplementary Material to ``On the Complexity of the Weighted fused Lasso''}

\section{Proof of Theorem \ref{th:spring_representation}}

We begin by reviewing Moreau's identity.
Let $f(x)$ be a closed proper convex function and let
\begin{equation}\label{eq:form_of_gen_prox_ope}
F(y) = \arg \min_x f(x) + \frac{1}{2}\|x-y\|^2_2
\end{equation}
be its proximal operator.
Let 
\begin{equation}
\hat{f}(x) = \sup_z \langle z,x\rangle - f(z)
\end{equation}
be the Fenchel dual of $f$ and let
\begin{equation}
\hat{F}(y) = \arg \min_x \hat{f}(x) + \frac{1}{2}\|x-y\|^2_2
\end{equation}
 be its proximal operator.
 Moreau's identity is
\begin{equation}\label{eq:moreide}
F(y) = y - \hat{F}(y).
\end{equation}

\begin{proof}[Proof of Theorem \ref{th:spring_representation}]
If $f(x) = \sum^{n-1}_{t=1} \gamma \alpha_t |x_{t+1}-x_t|$, 
then $x^*(\gamma) = F(y)$. Therefore, Theorem \ref{th:spring_representation} amounts to an
optimization problem to compute $F^*(y)$, from which
we can compute $F$, and hence $x^*$, using \eqref{eq:moreide}.

We start by making a change of variables in
\begin{equation}\label{eq:fenchel_subs}
f^*(x) =  \sup_z \sum^n_{t=1} z_t x_t - \sum^{n-1}_{t=1} \alpha_t |z_{t+1}-z_t|.
\end{equation}
Let $h_0 = z_1$ and $h_t = z_{t+1} - z_t$, for $t \in [n-1]$. This implies that $z_{t+1} = \sum^t_{i=0} h_i$
and thus that $\sum^n_{t=1} z_t x_t = \sum^n_{t=1} \sum^{t-1}_{i=0} h_i x_t = \sum^{n-1}_{t=0} h_t (\sum^n_{i = t+1} x_i)
= \sum^{n-1}_{t=0} h_t u_t$, where we have defined
$u_t=\sum^n_{i = t+1} x_i$
Therefore, we can rewrite \eqref{eq:fenchel_subs} as
\begin{equation}\label{eq:fenchel_subs_transformed}
f^*(x) =  \sup_h \sum^{n-1}_{t=0} h_t u_t - \sum^{n-1}_{t=0} \gamma \alpha_t |h_t|,
\end{equation}
where we have extended $\alpha$ such that $\alpha_0=0$.
Problem \eqref{eq:fenchel_subs_transformed} breaks down into $n$ independent one-dimensional problems of the form 
\begin{equation}
\sup_{h_t} h_t u_t - \gamma \alpha_t |h_t|,
\end{equation}
that have solution
$0$ if $u_t \in [-\gamma\alpha_t,\gamma\alpha_t]$, and $\infty$ otherwise.
Therefore, $\hat{f}(x) = \infty$ if ${u}_t=\sum^n_{i = t+1} x_i \notin [-\gamma\alpha_t,\gamma\alpha_t]$ for some $t = 0,\cdots,n-1$, and $\hat{f}(x) =0$ otherwise.

We can now write
\begin{align}\label{eq:almost_final_spring_transf}
&F(y) = y - \hat{F}(y) = y - \arg \min_x \sum^{n}_{t=1} (x_t - y_t)^2 \\
&\text{ subject to } \sum^n_{i = t+1} x_i \in [-\gamma \alpha_t,\gamma \alpha_t],t=0,\cdots,n-1.\nonumber
\end{align}

Finally, we make the following change of variable in
\eqref{eq:almost_final_spring_transf}:
 $w_t = \sum^n_{i = t} (x_i - y_i)$ for all
$t \in [n+1]$, where we define $w_{n+1} =0$.
This implies that $x_i  = y_i - w_{i+1} + w_{i}$
for $i\in[n]$, and hence that
\begin{align}
x^*_i &= F(y)_i = y_i - \hat{F}(y)_i = y_i - (y_i - w^*_{i+1} + w^*_{i})\\
&= w^*_{i+1} - w^*_{i}
\end{align}
for all $i\in [n]$, where 
\begin{align}
&w^* = \arg \min_w \sum^{n}_{t=1} (w_{t+1} - w_t)^2 \\
&\text{ subject to } w_{t+1} - \sum^n_{i=t+1} y_i \in [-\gamma \alpha_t,\gamma \alpha_t],t=0,\cdots,n-1.\nonumber
\end{align}
\end{proof}

\section{Simple algorithm to compute critical points}

Before we introduce the algorithm, we need a fourth observation, in addition to the three observations already made in the main text.
For this observation to be valid,
we assume that
the $\alpha$'s are not in some 
set of measure zero on the space of all possible $\alpha$'s. The following technical results, whose proofs are standard, e.g. \cite{halmos2013measure,sundaram1996first}, will be used to extend these observations to any set of $\alpha$'s.
We include a self contained proof of these lemmas in Section \ref{sec:app_proof_of_technical_lemma}.
\begin{lemma}\label{th:measure_zero}
Let $x \in X \subset \mathbb{R}^k$, where $X$ has zero Lebesgue measure.
For any $\epsilon > 0$, there exists a point $y \notin X$ such that $\|x-y\| < \epsilon$.
\end{lemma}
\begin{lemma}\label{eq:cont_lemma}
The function $x^*(\gamma,\alpha)$ is continuous
at every point of the domain $(\gamma,\alpha) \geq 0$.
\end{lemma}
\begin{remark}
Since $w^*(\gamma,\alpha)_{t+1} = \tilde{y}_1 + \sum^{t}_{i=1} x^*(\gamma,\alpha)_i$, this also proves that
$w^*(\gamma,\alpha)$ is continuous at every point of its domain.
\end{remark}

{\bf 4)} We can assume,
without loss of generality, that at most one component 
of $w^*$ transitions from free to non-free, or vice versa, at each $\gamma$.
This follows from the fact that, for two indices $i$ and $j$  (or more)
to satisfy \eqref{eq:critical_entering_time} for the same $\gamma_c$,
the values of $\tilde{\alpha}_i$, $\tilde{\alpha}_j$,
$\tilde{\alpha}_{i^\triangleright}$, $\tilde{\alpha}_{i^\triangleleft}$, $\tilde{\alpha}_{j^\triangleright}$ and $\tilde{\alpha}_{i^\triangleleft}$ 
 must belong to some set $\zeta$ of measure zero (in the space of possible $\alpha$'s).
Using Lemma \ref{th:measure_zero} and Lemma \ref{eq:cont_lemma}, we can then extend our arguments made outside $\zeta$ to this set as well.

Our four observations allows us to describe a
simple direct algorithm to compute the path $w^*(\gamma)$.
Its interpretation is simple: start with $\gamma = 0$.
Increase $\gamma$ and, as the intervals $\tilde{y}_i + \gamma [-\tilde{\alpha}_i,\tilde{\alpha}_i]$ grow larger, keep track
of which points are touching either limit of its interval.
For each interval of values of $\gamma$ for which $B$ is fixed, we can use \eqref{eq:movement_dictated_by_extremes_2} to compute how each point moves, figure out which next point that will become free, or non-free, and hence compute the next  $B$. 

Let us be more precise \footnote{The following algorithm only describes how to compute 
$\{\gamma_i\}^{T-1}_{i=1}$, the critical points.
It can be modified
to compute $\{w^*(\gamma_i)\}^{T-1}_{i=1}$
with only a multiplying factor slow down in the run time.}.
\begin{enumerate}[leftmargin=*]
\item Start with $\gamma_1 = 0$,  $F = \{\}$ and iteration number $r = 1$.
All points are non-free.
\item At iteration $r$, use \eqref{eq:critical_entering_time} 
to find $\gamma$'s at which an
$i \notin F$ \emph{might} enter $F$
or at which an $i \in F$ \emph{might} leave $F$.
Store these values as $\gamma^{(1)}_c$,$\gamma^{(2)}_c$, ... and $\gamma'^{(1)}_c$, $\gamma'^{(2)}_c$, ...
respectively, where the upper index $(i)$ in $\gamma^{(i)}_c$ and  
$\gamma'^{(i)}_c$ refers to the index of the point that might enter or leave $F$. If point $i$ cannot leave $F$ in round $r$, then $\gamma^{(i)}_c = \infty$. If point $i$ cannot enter $F$ in round $r$, then $\gamma'^{(i)}_c = \infty$.

These values are not all critical points on $w^*(
\gamma)$, but
all critical points satisfy \eqref{eq:critical_entering_time}, and, in particular, the smallest value produced corresponds to the next critical point. \label{enu:homotopy_while_loop}
\item Append or remove from $F$ the index
\begin{equation}
i_c = \min \{ \min_{i:\gamma^{(i)}_c > \gamma_r}  \gamma^{(i)}_c , \min_{i:\gamma'^{(i)}_c > \gamma_r}  \gamma'^{(i)}_c\}.
\end{equation}
Set $\gamma_{r+1} = \gamma^{(i_c)}_c$ or
$\gamma_{r+1} = \gamma'^{(i_c)}_c$. Set
$\gamma^{(i)}_c = \infty$ or $\gamma'^{(i)}_c = \infty$,
the choice depending on whether there is a point leaving, or entering $F$. We do this last assignment to avoid a cycle of removing and adding the same point to $F$, ad infinitum, for the same value of $\gamma$, without the algorithm progressing.
\item Update $\gamma^{(j)}_c$ or $\gamma'^{(j)}_c$ for all $j \in \{i^\triangleleft,\dots,i-1,i+1,\dots,i^\triangleright\}$. All the other previously computed $\gamma^{(j)}_c$ and $\gamma'^{(j)}_c$ will be the same. For later purposes,
let us call by $k_r$ the number of values updated in this step.
\item Terminate if $B = \{1,n+1\}$, otherwise go back to step \ref{enu:homotopy_while_loop}.
\end{enumerate}

In the above algorithm, it is convenient to represent $B$ as a linked list
such that we can access its elements in the order
of their indices. We do not explicitly represent $F$. An element is in $F$
if it is not in the linked list $B$. 
This representation for $B$ also allows us to
loop over the elements in $F$ in order of their indices, by
looping over the elements of $B$ and, for 
any two consecutive elements in $B$, say $a < b$, 
looping over all indices $a+1,\dots,b-1$, which we know must be in $F$.
Given a point $i$, either in $B$ or in $F$, this allows us to 
determine, in $\mathcal{O}(1)$
steps, the indices $i^{\triangleleft}$ and
$i^{\triangleright}$.
We can also add and remove points from $B$ in
$\mathcal{O}(1)$ steps while keeping the list ordered.

If we use a binary minimum heap to keep track of the
minimum value of
across $\{\gamma^{(j)}_c\}$ and $\{\gamma'^{(j)}_c\}$,
we pay a computational cost of $\mathcal{O}(\log n)$ each time
we update $\gamma^{(j)}_c$ or $\gamma'^{(j)}_c$ for some $j$.
Therefore, the complexity of this algorithm is
\begin{equation}
\mathcal{O}\left(\left(\sum^T_{r=1} k_r \right) \log n\right)
\end{equation}
where, we recall, $T$ is the number of critical points in the path
$\{w^*(\gamma)\}$, and $n$ is the number of components in the input $y$. In the particular case of 1FL, where $\alpha_i = 1\; \forall i$,
we have $T \leq n$ (See e.g. Theorem \ref{th:linear_complexity_1FL}) and points only enter $F$, they never leave $F$. Therefore, we only
need to use \eqref{eq:critical_entering_time} for
points in $B$ and thus $k_r =2 \; \forall r$. This leads us to
$\mathcal{O}(n\log n)$, just like in \cite{hoefling2010path}.

\section{Proof of technical Lemma \ref{th:measure_zero} and Lemma \ref{eq:cont_lemma}}\label{sec:app_proof_of_technical_lemma}

\begin{proof}[Proof of Lemma \ref{th:measure_zero}]
The proof follows by contradiction. 
Assume that there exists $\epsilon > 0$ such that, for all $y$ with $\|x-y\| < \epsilon$, we have $y \in X$.
$X$ contains a ball of size $\epsilon/2$, and thus has non-zero measure.
\end{proof}

\begin{proof}[Proof of Lemma \ref{eq:cont_lemma}]

Since $(\gamma {\alpha}_1,\dots,\gamma {\alpha}_{n+1})$ is continuous as a function of
$(\gamma, \alpha)$, it follows that, to prove that 
$x^*$ is continuous as a function of $(\gamma, \alpha)$,
we only need to prove that $x^*$ is continuous as a function of $\alpha$.

We now assume $\gamma$ fixed, and for simplicity write $x^*(\gamma)$
as $x^*$ or $x^*(\alpha)$. The same goes for all the variables introduced below.
The proof proceeds in two steps.
First we show that $x^*$ can be obtained from a linear
transformation of a point $z^*$ defined as the
point in the convex polytope
$\mathcal{P}(\alpha) = \{z \in \mathbb{R}^n:  
|\tilde{y}_i - \sum^i_{r=1} z_r| \leq \gamma \tilde{\alpha}_{i-1} \text{ for all } i\in[n+1] \}$ that is closest to the origin.
Second we show that $z^*$ is continuous in $\alpha$.

To establish the first step, just make the change of variable
$z_1 = w_1$, $z_{i+1} = w_{i+1} - w_i$, $i > 1$
 in \eqref{eq:spring_interpretation}.
 
To establish the second step, we first make three observations.
(Obs. 1) If $|\alpha' - \alpha| < \delta$, then for any $z' \in \mathcal{P}(\alpha')$, there exists $z \in \mathcal{P}(\alpha)$ such that
$\|z-z'\| < \epsilon_1(\delta)$, where $\epsilon_1(\delta)$ converges to zero as $\delta$ converges to zero.
(Obs. 2) If $|\alpha' - \alpha| < \delta$, then $\|z^*(\alpha')\| < 
\|z^*(\alpha)\| + \epsilon_2(\delta)$, where $\epsilon_2(\delta)$ converges to zero as $\delta$ converges to zero.
(Obs. 3) If $z\in \mathcal{P}(\alpha)$ is such that $\|z\| < \|z^*\| + \delta$, then $\|z - z^*\| < \epsilon_3(\delta)$, where
$\epsilon_3(\delta)$ converges to zero as $\delta$ converges to zero.

{
Obs. 1 follows because the faces of the polytope change continuously with $\alpha$. Obs. 2 follows from
Obs. 1, since, if $z\in\mathcal{P}(\alpha')$ is the closet point to
$z^*(\alpha)$, then, by Obs. 1, $\|z^*(\alpha)\| > \|z\| - \epsilon_1(\delta) \geq 
\|z^*(\alpha')\| - \epsilon_1(\delta)$.
Obs. 3 is trivial when the origin is in the interior
of $\mathcal{P}(\alpha)$, so we focus on the other case.
Let $d = \sup \|z - z^*\| $ be such that $z \in \mathcal{P}$, and $\|z\| < \|z^*\| +  \delta$, and let $d' = \sup \|z - z^*\| $ be such that $z \in \mathcal{H}$,  and $\|z\| < \|z^*\| +  \delta$, where $\mathcal{H}$ is the half-plane defined by 
$z^\top z^*\geq \|z^*\|^2$. Since $\mathcal{P}(\alpha)$ is convex, $\mathcal{P}(\alpha) \subseteq \mathcal{H}$ and thus $d \leq d' < 
\sqrt{(\|z^*\|+ \delta)^2 - \|z^*\|^2}=
\sqrt{\delta} \sqrt{2 \|z^*\| + \delta}$, where the last inequality follows by simple geometry.
Therefore, $\|z - z^*\| < \sqrt{\delta} \sqrt{2 \|z^*\|+ \delta}$,
which converges to zero as $\delta$ converges to zero.}

Now let $|\alpha' - \alpha| < \delta$
and let $q \in \mathcal{P}(\alpha)$ be the closet point to $z^*(\alpha')$. By Obs. 1, $\|z^*(\alpha') - q\| < \epsilon_1(\delta)$ and thus $\|q\| < \|z^*(\alpha')\| + \epsilon_1(\delta)$.
By Obs. 2 , $\|q\| <  \|z^*(\alpha)\| + \epsilon_2(\delta)+\epsilon_1(\delta)$. By Obs. 3, $\| q -z^*(\alpha)\| < \epsilon_3(\epsilon_2(\delta)+\epsilon_1(\delta))$.
We can finally write, 
$\|z^*(\alpha') -z^*(\alpha)\| = \|z^*(\alpha') - q+ q -z^*(\alpha)\|\leq \|z^*(\alpha') - q\|  + \| q -z^*(\alpha)\|
\leq \epsilon_1(\delta)  + \epsilon_3(\epsilon_2(\delta)+\epsilon_1(\delta))$. This proves continuity since the right hand side converges to zero as $\delta$ converges to zero.

\end{proof}

\end{document}